\documentclass[letterpaper,12pt,onecolumn]{ieeeconf}

\IEEEoverridecommandlockouts                              % This command is only needed if 
\usepackage[T1]{fontenc}
\usepackage{lmodern}
\usepackage{comment}
\usepackage{microtype}

\usepackage{amsmath,amssymb,mathtools} %amsthm
\usepackage{algorithmic}
\usepackage{algorithm}
\usepackage{array}
\usepackage[caption=false,font=normalsize,labelfont=sf,textfont=sf]{subfig}
\usepackage{textcomp}
\usepackage{stfloats}
\usepackage{url}
\usepackage{verbatim}
\usepackage{graphicx}
\usepackage{cite}
\usepackage{enumerate}
\usepackage{balance}
\usepackage{siunitx}

\usepackage{tikz}

\newcommand{\benchmarkmark}{%
  \tikz[baseline=(char.base)]{
    \node[shape=circle,draw,inner sep=0.15ex] (char) {\scriptsize$\bullet$};
  }%
}

\newcommand{\proposedmark}{%
  \tikz[baseline=-0.45ex] \fill (0,0) circle (0.75ex);%
}

\usepackage{xcolor}
\usepackage[hidelinks]{hyperref} % hidelinks keeps the text black for IEEE

\usepackage{xurl} 
\usepackage{amsthm}

\newtheorem{proposition}{Proposition}

\newtheorem*{remark}{Remark}

\renewcommand{\baselinestretch}{1.02}

\usepackage[acronym]{glossaries}

\makeglossaries

\newacronym{res}{RES}{Renewable Energy Sources}
\newacronym{bess}{BESS}{Battery Energy Storage Systems}
\newacronym{pv}{PV}{Photovoltaic}
\newacronym{dr}{DR}{Demand Response}
\newacronym{lfm}{LFM}{Local Flexibility Markets}
\newacronym{ems}{EMS}{Energy Management Systems}
\newacronym{ev}{EV}{Electric Vehicles}
\newacronym{iso}{ISO}{Independent System Operators}
\newacronym{scm}{SCM}{Synthetic Control Method}
\newacronym{gan}{GAN}{Generative Adversarial Network}
\newacronym{gnn}{GNN}{Graph Neural Networks}
\newacronym{dnn}{DNN}{Dense Neural Networks}
\newacronym{moe}{MoE}{Mixture of Experts}
\newacronym{lstm}{LSTM}{Long Short-Term Memory Models}
\newacronym{soc}{SOC}{State of Charge}
\newacronym{std}{STD}{Standard Deviation}
\newacronym{mse}{MSE}{Mean Squared Error}
\newacronym{hvac}{HVAC}{Heating, Ventilation, and Air Conditioning}
\newacronym{mlp}{MLP}{Multilayer Perceptron}

\title{\LARGE \bf
A Generalized Synthetic Control Method \\ for Baseline Estimation in Demand Response Services
}

\author{Jonas Sievers$^{1}$ and Mardavij Roozbehani$^{2}$% <-this % stops a space
\thanks{$^{1}$Jonas Sievers is with Institute for Data Processing and Electronics, Karlsruhe Institute of Technology, 76344 Eggenstein-Leopoldshafen, Germany {\tt\small jonas.sievers@kit.edu}}%
\thanks{$^{2}$Mardavij Roozbehani with the Laboratory for Information and Decision Systems, Massachusetts Institute of Technology, Boston, MA 02139, USA {\tt\small mardavij@mit.edu}}%
}

\begin{document}

\maketitle
\thispagestyle{plain}
\pagestyle{plain}

%%%%%%%%%%%%%%%%%%%%%%%%%%%%%%%%%%%%%%%%%%%%%%%%%%%%%%%%%%%%%%%%%%%%%%%%%%%%%%%%

\begin{abstract}
Baseline estimation is critical to Demand Response (DR) settlement in electricity markets, yet existing machine learning methods remain limited in predictive performance, while methodologies from causal inference and counterfactual prediction are still underutilized in this domain. We introduce a \emph{Generalized Synthetic Control Method} that builds on the classical \gls{scm} from econometrics. While SCM provides a powerful framework for counterfactual estimation, classical SCM remains a static estimator: it fits the treated unit as a combination of contemporaneous donor units and therefore ignores predictable temporal structure in the residual error. We develop a generalized SCM framework that transforms baseline estimation into a dynamic counterfactual prediction problem by augmenting the donor representation with exogenous features, lagged treated load, and selected lagged donor signals. This enriched representation allows the estimator to capture autoregressive dependence, delayed donor-response patterns, and error-correction effects beyond the scope of standard SCM. The framework further accommodates nonlinear predictors when linear weighting is inadequate, with the greatest benefit arising in limited-data settings. Experiments on the Ausgrid smart-meter dataset show consistent improvements over classical SCM and strong benchmark methods, with the dominant performance gains driven by dynamic augmentation.   

\end{abstract}

\section{INTRODUCTION}

Flexibility services such as \gls{dr} and \glspl{lfm} are becoming increasingly important for maintaining grid stability in power systems with growing shares of volatile \gls{res}. While \gls{dr} programs incentivize consumers to adjust electricity consumption in response to grid conditions, \glspl{lfm} provide market-based mechanisms for procuring local flexibility to relieve congestion. 

A central challenge in both \gls{dr} and \gls{lfm} is baseline estimation, i.e., the counterfactual consumption that would have occurred absent a flexibility incentive \cite{muthirayan_2018, satchidanandan_2023}. Because delivered flexibility is measured as the deviation from this baseline, accurate and transparent estimation is essential for fair settlement and efficient market design. Inaccurate baselines can distort remuneration, weaken participation incentives, and create opportunities for strategic manipulation through deliberate load shifting \cite{ellman_customer_2019}.

\begin{figure}[!t]
    \centering
    \includegraphics[width=4 in]{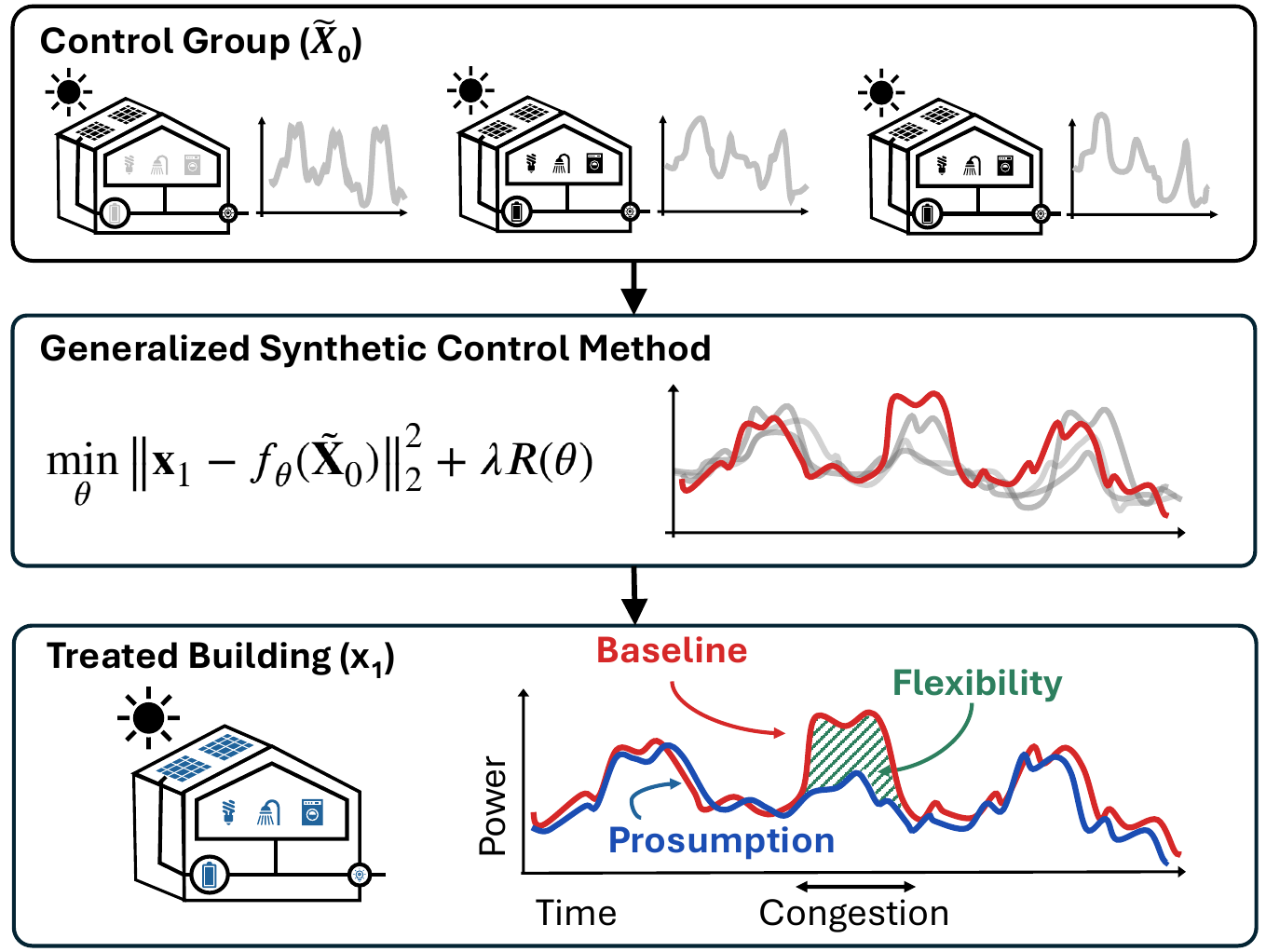}
    \caption{Architectural overview of the Generalized Synthetic Control Method for baseline estimation in residential flexibility services.}
    \label{fig:Overview_SCM}
\end{figure}

\subsection{Related Work} \label{sec_relatedWork}
Recent approaches to customer baseline estimation can be broadly grouped into averaging, regression, clustering-based, and deep learning methods (\autoref{tab_related_work}). Averaging-based methods, including Exponential Moving Average and X-of-Y techniques, are widely used in practice because of their simplicity \cite{pati_methodologies_2020}. However, their reliance on historical averages limits accuracy when consumption patterns shift rapidly.

Regression-based methods improve adaptability by modeling consumption as a function of explanatory variables. Early approaches used linear and polynomial regression \cite{pati_methodologies_2020}, with Ridge or Lasso regularization. More flexible models include Support Vector Machines \cite{campodonico_avendano_assessing_2023}, Random Forests, and XGBoost \cite{pati_methodologies_2020}.

Deep learning has further advanced baseline estimation by capturing nonlinear dependencies in load profiles through architectures such as \gls{dnn}, \gls{lstm}, and \glspl{gan} with attention mechanisms \cite{zhang_cbl_2024, wang_residential_2020, wang_customer_2024}. Like other forecasting-based approaches, these models estimate counterfactual consumption from pre-intervention information, extrapolating historical load and covariates under the assumption of no flexibility activation.

In clustering-based approaches, the baseline is inferred from historical trends together with contemporaneous observations from a peer group selected to resemble the treated building. Methods that use reference-building information during the intervention period are closer in spirit to the Synthetic Control Method. Common techniques include K-Means, DBScan, Self-Organizing Maps, and frequency-based clustering based on the Discrete Fourier Transform \cite{schwarz_building_2020}.

The \gls{scm}, originally introduced by Abadie et al. \cite{abadie_economic_2003, abadie_synthetic_2010}, extends this peer-based logic by replacing discrete group selection with continuous weight optimization. Rather than relying on a pre-clustered subset, \gls{scm} constructs a synthetic counterfactual as a weighted combination of units in a donor pool that are not exposed to the intervention. In the context of baseline estimation, this amounts to combining profiles from similar non-participating buildings with weights chosen to match the treated building’s historical consumption. Owing to its causal interpretability and flexibility, \gls{scm} has been widely applied in economics, public health, social sciences, and engineering \cite{abadie_using_2021}.

\begin{table*}
{\footnotesize
\begin{center}
\caption{Concept Matrix for the Literature on Baseline Estimation in Flexibility Services.}
\label{tab_related_work}
{
\setlength{\tabcolsep}{2pt}  % default ~6pt
\renewcommand{\arraystretch}{1.2} % optional: slightly tighter rows

\begin{tabular}{ll|lccccc}
\hline
Reference & Year & Focus & Average & Regression & Deep L. & Cluster & Optimization \\
\hline
\cite{peng_aggregated_2024} & 2024 & Aggregated BE with a Graph Neural Network & \benchmarkmark & \benchmarkmark & \benchmarkmark\ \proposedmark & \benchmarkmark\ \proposedmark & \\ 
\hline
\cite{wang_residential_2020} & 2020 & Aggregated BE with Stacked Autoencoders & \benchmarkmark & \benchmarkmark & \benchmarkmark\ \proposedmark & \benchmarkmark\ \proposedmark & \\
\hline
\cite{wang_customer_2024} & 2024 & BE with an Attention-based GAN & \benchmarkmark & \benchmarkmark & \benchmarkmark\ \proposedmark & \benchmarkmark\ \proposedmark & \\  
\hline
\cite{qian_residential_2024} & 2024 & BE with a Diffusion Model & \benchmarkmark & \benchmarkmark & \benchmarkmark\ \proposedmark & \benchmarkmark\ \proposedmark & \\ 
\hline
\cite{pati_methodologies_2020} & 2020 & BE with Dense Neural Networks & \benchmarkmark & \benchmarkmark & \proposedmark & & \\
\hline
\cite{campodonico_avendano_assessing_2023} & 2023 & BE Evaluation of various ML Models & \benchmarkmark & \benchmarkmark & & & \\
\hline
\cite{zhou_robust_2022} & 2022 & BE with Mixed Effect Regression & & \benchmarkmark\ \proposedmark & & & \\
\hline
\cite{ge_spatio-temporal_2022} & 2022 & BE with K-means and Lasso Regression & \benchmarkmark & \benchmarkmark\ \proposedmark & & \benchmarkmark\ \proposedmark & \\
\hline
\cite{ying_customers_2023} & 2023 & BE with K-means and Weighted Averaging & \benchmarkmark\ \proposedmark & & & \proposedmark & \\
\hline
\cite{zhang_cbl_2024} & 2024 & BE with K-means and LSTM & & & \proposedmark & \proposedmark & \\
\hline
\cite{schwarz_building_2020} & 2020 & BE with Frequency-based Clustering and Averaging & \benchmarkmark\ \proposedmark & \proposedmark & & \benchmarkmark\ \proposedmark & \\
\hline
This paper & 2025 & BE with Advanced Synthetic Control Method & \benchmarkmark & \benchmarkmark & \benchmarkmark\ \proposedmark & \benchmarkmark & \proposedmark \\
\hline
\end{tabular}
}
\end{center}
}
\begin{minipage}{\textwidth}
Where \proposedmark\ represents a proposed method in the respective category, and \benchmarkmark\ indicates its inclusion as a benchmark. Abbreviations: Baseline Estimation (BE); Generative Adversarial Network (GAN); Machine Learning (ML); Long Short-Term Memory Model (LSTM); Deep Learning (Deep L).
\end{minipage}
\end{table*}

\subsection{Contribution}

We propose a Generalized Synthetic Control framework for baseline estimation (\autoref{fig:Overview_SCM}). To the best of our knowledge, this is the first work to apply counterfactual prediction ideas from the causal-inference literature, and specifically Synthetic Control, to DR baseline estimation.  Our main methodological contribution is to extend classical SCM from a static estimator based on contemporaneous donor loads to a dynamic counterfactual predictor built from an augmented representation including exogenous features, lagged treated load, and selected lagged donor signals. This reformulation enables SCM to incorporate temporal information that is excluded by its standard static formulation. We further generalize the framework to admit nonlinear predictors and show empirically that their benefit is most pronounced in data-constrained settings or when donor pools are small. Finally, using real-world smart-meter data, we demonstrate consistent improvements over classical SCM and strong benchmark methods, with the largest gains attributable to dynamic augmentation while retaining the interpretability that makes SCM attractive for market applications.

\section{METHODOLOGY} \label{sec_methodology}
%This section presents the methodology, including control group selection and nonlinear extensions. Scalars are denoted by ($y$), vectors by ($\mathbf{y}$), and matrices by ($\mathbf{Y}$).

\subsection{Baseline Estimation Using the Synthetic Control Method}

We estimate the DR baseline using a synthetic control formulation, in which the untreated consumption of a participating building is approximated by a weighted combination of non-participating buildings.

Let \(y_{jt} \in \mathbb{R}\) denote the electricity consumption of building \(j \in \{1,\dots,J{+}1\}\) at time \(t \in T_0 \cup T_1\), where \(T_0\) and \(T_1\) denote the pre and post intervention periods, respectively. Building \(j=1\) is the treated unit, i.e., the building that received the flexibility payment, while the remaining \(J\) buildings serve as donor units and receive no flexibility payments. During congestion (\(t \in T_1\)), the delivered flexibility, or treatment effect, is defined as the deviation of observed consumption \(y_{1t}\) from the estimated baseline \(\hat y_{1t}\) \cite{abadie_using_2021}:
\begin{equation}\label{eq_scm_flex}
    \hat{\tau}_{1t} = \hat{y}_{1t} - y_{1t}, \quad t \in T_1.
\end{equation}

The baseline is estimated as a weighted combination of donor consumption:
\begin{equation}\label{eq_scm_be}
    \hat{y}_{1t} = \sum_{j=2}^{J+1} w_j\, y_{jt}, \quad t \in T_1,
\end{equation}
where the weight vector \(\mathbf w = (w_2,\dots,w_{J+1})^\top\) is chosen to best reproduce the treated unit's pre-congestion profile. Let \(\mathbf x_1 \in \mathbb{R}^{|T_0|}\) denote the treated building's consumption over \(T_0\), and let \(\mathbf X_0 \in \mathbb{R}^{|T_0| \times J}\) denote the corresponding matrix of donor consumption. In the classical SCM formulation, the weights are obtained by solving
\begin{equation}\label{eq_scm_w}
    \mathbf w^*
    =
    \arg\min_{\mathbf w}\;
    \|\mathbf x_1 - \mathbf X_0 \mathbf w\|_2^2
    \quad
    \text{s.t.}
    \quad
    w_j \ge 0,\;
    \sum_{j=2}^{J+1} w_j = 1.
\end{equation}

The simplex constraints promote interpretability and regularization: each coefficient \(w_j\) can be interpreted as the contribution of donor \(j\) to the synthetic control. In our setting, however, we relax the non-negativity constraint and allow \(w_j \in \mathbb{R}\), which permits the model to capture both positive and negative correlations among building-level consumption profiles. To control variance and reduce overfitting, we instead impose explicit \(\ell_2\)-regularization:
\begin{equation}\label{eq_scm_general}
    \mathbf w^*
    =
    \arg\min_{\mathbf w}\;
    \|\mathbf x_1 - \mathbf X_0 \mathbf w\|_2^2
    + \lambda \|\mathbf w\|_2^2
    \quad
    \text{s.t.}
    \quad
    \sum_{j=2}^{J+1} w_j = 1.
\end{equation}

\subsection{Augmented Synthetic Control Design}
\begin{figure*}[!t]
    \centering
    \subfloat[]{\includegraphics[width=3.2in]{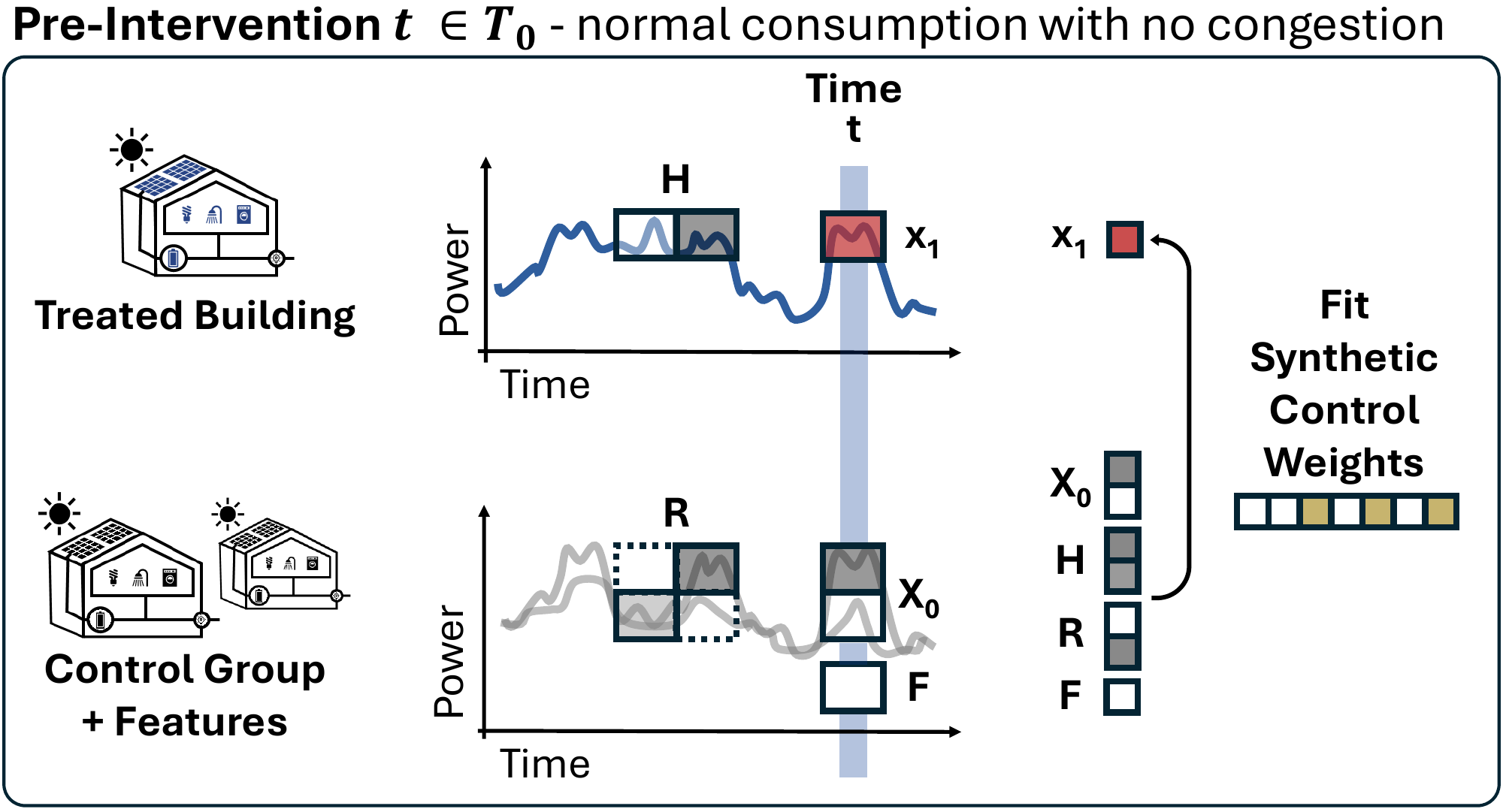}%
    \label{fig2a_pre}}
        \hfil
    \subfloat[]{\includegraphics[width=3.2in]{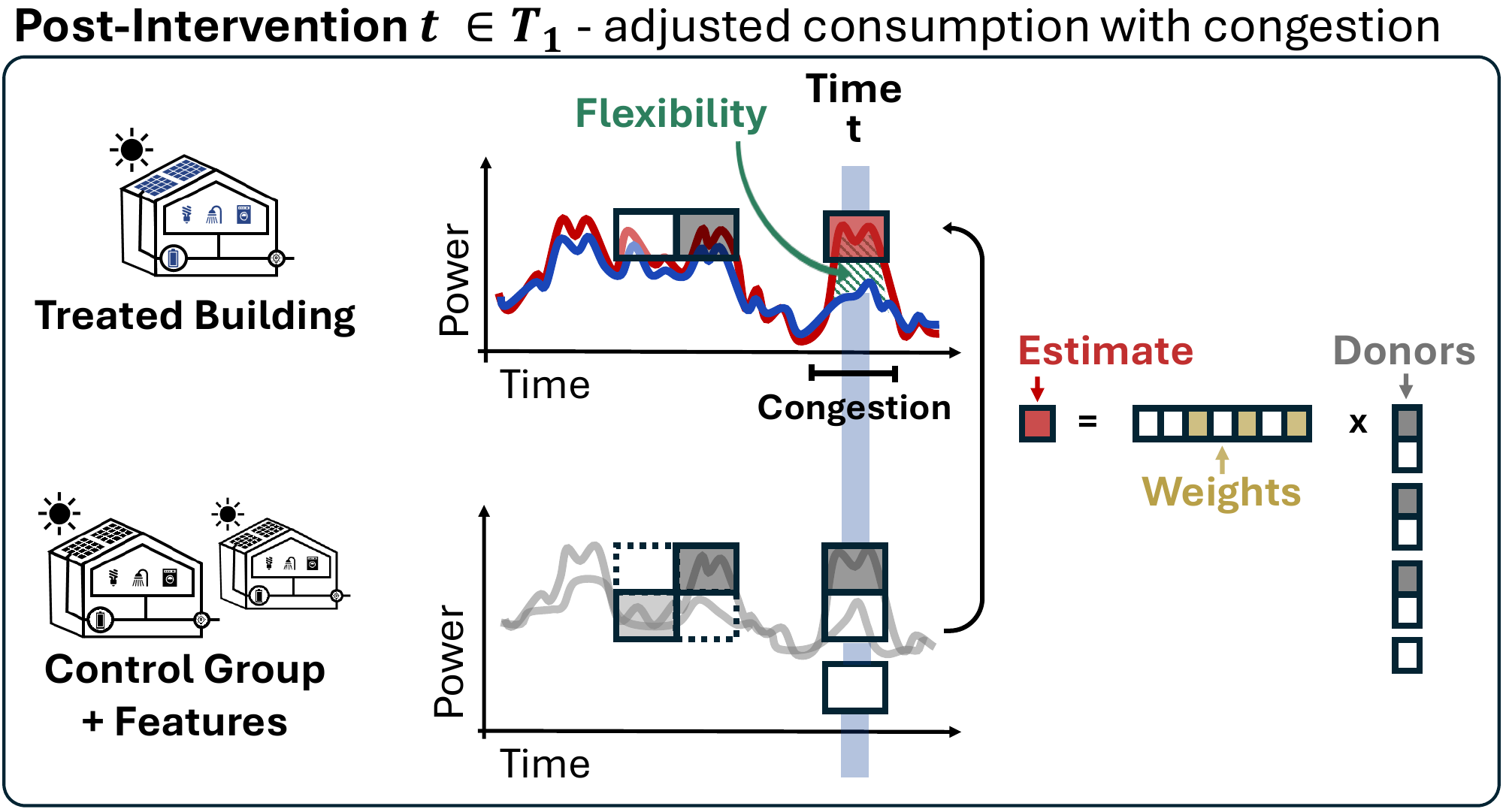}%
    \label{fig2b_post}}
    \caption{Overview of the proposed baseline estimation framework for flexibility provision in demand response programs: (a) estimation of the synthetic control weights from the extended control group during the pre-intervention period, and (b) reconstruction of the baseline during congestion events.}
    \label{fig:scm_fit}
\end{figure*}

The accuracy of counterfactual estimation in \gls{scm} depends on the donor pool’s ability to reproduce both observed predictors and latent factors of the treated unit. A key limitation of standard \gls{scm} is that this representation is static: it relies only on contemporaneous donor outcomes and therefore does not exploit temporal structure in the residual mismatch between the treated unit and its synthetic control. In practice, the treated unit need not lie exactly in the linear span of the donors, and the resulting error may exhibit serial dependence or delayed dependence on common demand drivers. To address this, we augment \(\mathbf{X}_0 \in \mathbb{R}^{|T_0| \times J}\) with exogenous covariates, lagged treated outcomes, and selected lagged donor outcomes.\footnote{See \ref{ap:scm_theory} for theoretical justification and discussion.} Specifically, we introduce three additional feature blocks (\autoref{fig:scm_fit}):
\begin{enumerate}[(i)]
    \item \textbf{Exogenous features} \(\mathbf{F} \in \mathbb{R}^{|T_0| \times M}\). For each \(t \in T_0\), the row \(\mathbf{f}_t^\top\) contains \(M\) observed covariates, such as temperature, humidity, and calendar indicators, that serve as proxies for external demand drivers.

    \item \textbf{Lagged treated outcomes} \(\mathbf{H} \in \mathbb{R}^{|T_0| \times L}\). For each \(t \in T_0\), the row \(\mathbf{h}_t^\top\) contains the \(L\) most recent consumption values of the treated unit,
    \[
    \mathbf{h}_t^\top = (y_{1,t-1}, y_{1,t-2}, \dots, y_{1,t-L}),
    \]
    so that \(\mathbf{H}_{t,\ell} = y_{1,t-\ell}\) for \(\ell = 1,\dots,L\).

    \item \textbf{Selected lagged donor outcomes} \(\mathbf{R} \in \mathbb{R}^{|T_0| \times J}\). For each donor \(j \in \{2,\dots,J+1\}\), we select the lag \(k_j^*\) with the largest absolute Pearson correlation with the treated series over the pre-intervention period:
    \begin{equation}\label{eq_pearsonCorr}
        k_j^* \in \arg\max_{1 \le k \le L}
        \left|
        \operatorname{corr}\big(
        \{y_{1t}\}_{t\in T_0},
        \{y_{j,t-k}\}_{t\in T_0}
        \big)
        \right|.
    \end{equation}
    For each \(t \in T_0\), the row \(\mathbf{r}_t^\top\) is
    \[
    \mathbf{r}_t^\top =
    \big(
    y_{2,t-k_2^*},\,
    y_{3,t-k_3^*},\,
    \dots,\,
    y_{J+1,t-k_{J+1}^*}
    \big).
    \]
    Thus, each donor contributes one lagged series chosen to maximize alignment with the treated unit while keeping the feature dimension manageable.
\end{enumerate}

The augmented feature matrix is then
\begin{equation} \label{eq_ext_matrix}
\tilde{\mathbf{X}}_0
=
\left[
\mathbf{X}_0 \quad
\mathbf{F} \quad
\mathbf{H} \quad
\mathbf{R}
\right]
\in \mathbb{R}^{|T_0| \times D},
\end{equation}
where \(D = J + M + L + J \). The synthetic control coefficients are estimated using \eqref{eq_scm_general} with \(\mathbf{X}_0\) replaced by \(\tilde{\mathbf{X}}_0\).

An important distinction from standard \gls{scm} arises when lagged treated outcomes are included. In classical \gls{scm}, post-intervention predictions depend only on donor information and can therefore be evaluated over any post-intervention horizon. Once \(\mathbf{H}\) is included, however, multi-step counterfactual prediction is valid only if post-intervention lagged treated inputs are generated recursively from previously estimated counterfactual outcomes rather than realized treated outcomes. Otherwise, evaluation should be restricted to the one-step-ahead prediction at \(T_0+1\).

\subsection{Integrating Nonlinear Predictors into the Synthetic Control Framework}

Classical \gls{scm} estimates the counterfactual baseline as a linear combination of donor features. This can be restrictive when the treated unit is not well approximated by a linear predictor, even after augmentation. We therefore replace the linear map by a function \(f_{\theta}:\mathbb{R}^{D}\to\mathbb{R}\) that maps the \(t\)-th row of the augmented feature matrix \(\tilde{\mathbf{X}}_0\), denoted \(\tilde{\mathbf{x}}_{0t}\), to the treated outcome \(y_{1t}\).

The model parameters are estimated on the pre-treatment window \(T_0\) by minimizing the regularized mean squared error
\begin{equation}\label{eq_scm_nn1}
    \theta^* \in \arg\min_{\theta}\;
    \frac{1}{|T_0|}\sum_{t\in T_0}\big(y_{1t}-f_{\theta}(\tilde{\mathbf{x}}_{0t})\big)^2
    + \lambda R(\theta),
\end{equation}
where \(R(\theta)\) denotes a regularization term. During congestion periods \(t\in T_1\), the counterfactual baseline is then obtained by forward evaluation of the trained model:
\begin{equation}\label{eq_scm_nn2}
    \hat{y}_{1t} = f_{\theta^*}(\tilde{\mathbf{x}}_{0t}), \qquad t\in T_1.
\end{equation}

Equations \eqref{eq_scm_nn1}--\eqref{eq_scm_nn2} generalize \gls{scm} by allowing nonlinear dependence on the augmented donor representation. In principle, \(f_\theta\) may be implemented using more expressive sequence models, such as \gls{lstm}s or Transformers. In this paper, however, our goal is to introduce the general framework rather than optimize model architecture, and we therefore use a basic \gls{mlp}. The classical linear formulation is recovered as the special case \(f_{\theta}(\mathbf{z})=\mathbf{z}^\top\theta\), with \(\mathbf{z}=\tilde{\mathbf{x}}_{0t}\).

\section{RESULTS} \label{sec_results} 
The experimental setup is described in \ref{ap:scm_exp_setup}. The evaluation is based on the Ausgrid dataset \cite{ratnam_residential_2017}, extended with simulated battery operation and aligned meteorological features. Results are reported for 50 treated buildings, each evaluated against a donor pool of 299 control buildings. 

\subsection{Linear Synthetic Control for Baseline Estimation} \label{sec_results_linear}
We begin by benchmarking the linear \gls{scm}, which serves as the foundation for later extensions. \autoref{fig2_bench} presents the \gls{mse} for averaging (grey), regression (blue), and clustering (yellow) benchmarks, compared with our linear \gls{scm} (red).

\begin{figure}[!t]
    \centering
    \framebox{\parbox{5 in}{\includegraphics[width=5 in]{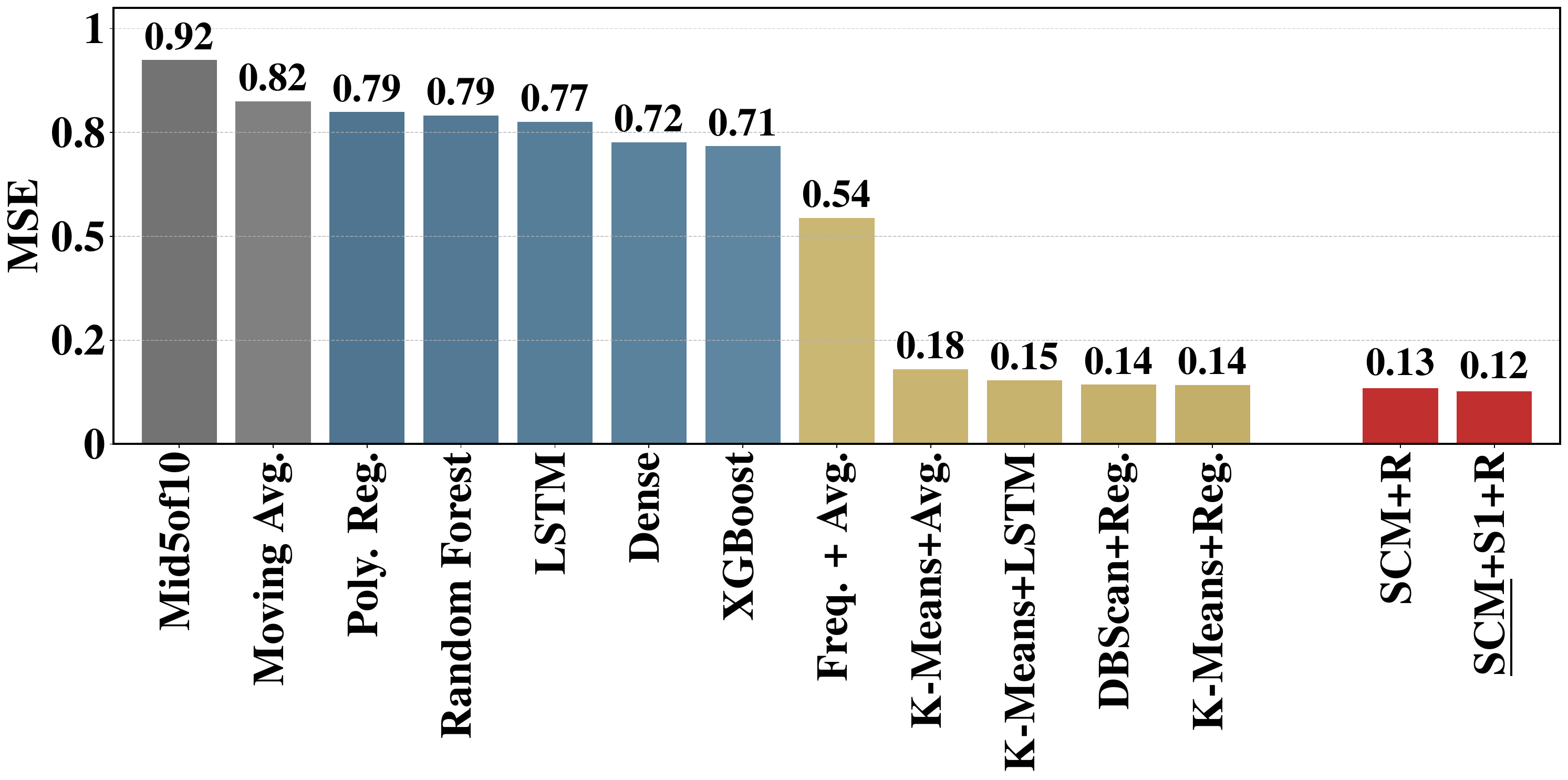}}}
    \caption{Baseline Estimation over 50 Buildings, comparing Averaging (grey), Regression (blue), Clustering (yellow), and Synthetic Control (red) Methods.}
    \label{fig2_bench}
\end{figure}

\begin{table}[h]
\centering
\caption{Accuracy of Linear Synthetic Control Variants and the Best Benchmark over 50 Residential Buildings.}
\label{tab2_lscm}
\footnotesize
\renewcommand{\arraystretch}{1.2}
\setlength{\tabcolsep}{14pt} % default is about 6pt, increase for a wider table
\begin{tabular}{|l|c|c|c|c|c|}
\hline
\textbf{Scenario} & \textbf{MSE} & \textbf{Min} & \textbf{Max} & \textbf{STD} & \textbf{Diff} \\ \hline
K-Means+Lasso   & 0.1426 & 0.0303 & 0.3610 & 0.0751 &               \\ 
SCM+R           & 0.1351 & 0.0273 & 0.4077 & 0.0869 & 5.26\% \\ 
$\underline{\text{SCM}}$+S1+R & 0.1274 & 0.0243 & 0.3393 & 0.0815 & 10.66\% \\ 
$\underline{\text{SCM}}$+R      & 0.1275 & 0.0246 & 0.3366 & 0.0818 & 10.60\% \\ \hline
\end{tabular}%

\vspace{4pt} % This adds the gap precisely between the table and the note
\begin{minipage}{\columnwidth}
\scriptsize
\textbf{Abbreviations:} K-Means + Lasso denotes K-Means clustering with Lasso regression; $\underline{\text{SCM}}$ unconstrained SCM; +S1 adds sum-to-one; +R regularization.
\end{minipage}
%\vspace{-10pt}
\end{table}

Among the averaging-based baselines, Mid5of10 \cite{pati_methodologies_2020} and Moving Average \cite{pati_methodologies_2020} yield \gls{mse} values of 0.9243 and 0.8250, respectively. Regression-based methods show modest improvements: Polynomial Regression \cite{chen_short-term_2017} achieves an \gls{mse} of 0.7992, followed by Random Forest \cite{campodonico_avendano_assessing_2023} at 0.7914, \gls{lstm} \cite{zhang_cbl_2024} at 0.7756, \gls{dnn} \cite{pati_methodologies_2020} at 0.7268, and XGBoost \cite{campodonico_avendano_assessing_2023} at 0.7170. These reductions are consistent across models, but remain within statistical variation, with standard deviations ranging from 1.6095 (XGBoost) to 1.8329 (Random Forest).

In our experiments, clustering-based methods that exploit information from similar buildings consistently outperform baselines based only on the treated building. K-means and DBSCAN perform comparably, while frequency-based clustering yields slightly higher errors. Among these variants, K-means with simple averaging attains an \gls{mse} of 0.1800, K-means with an \gls{lstm} reduces this to 0.1534, and K-means with Lasso regression achieves the lowest value, 0.1426.

\autoref{tab2_lscm} compares the best benchmark, K-means with Lasso regression, against three regularized linear \gls{scm} variants. SCM+R, which enforces both sum-to-one and non-negativity together with regularization, achieves an \gls{mse} of 0.1351, improving on the benchmark by \SI{5.26}{\%}. The $\underline{\text{SCM}}$+S1+R variant, which retains only the sum-to-one constraint, achieves the lowest \gls{mse}, 0.1274, corresponding to a \SI{10.66}{\%} improvement and reducing the minimum error from 0.0303 to 0.0243 and the maximum from 0.3610 to 0.3393. Finally, the unconstrained $\underline{\text{SCM}}$+R variant attains an \gls{mse} of 0.1275, corresponding to a \SI{10.60}{\%} improvement.

These results show that regularized linear \gls{scm} outperforms the strongest benchmark while retaining interpretable weight structures suitable for demand response settlement.

\subsection{Effect of the Augmented Control Group} \label{sec_results_generalized_control_group}

\begin{figure}[!t]
    \centering
    \framebox{\parbox{5 in}{\includegraphics[width=5 in]{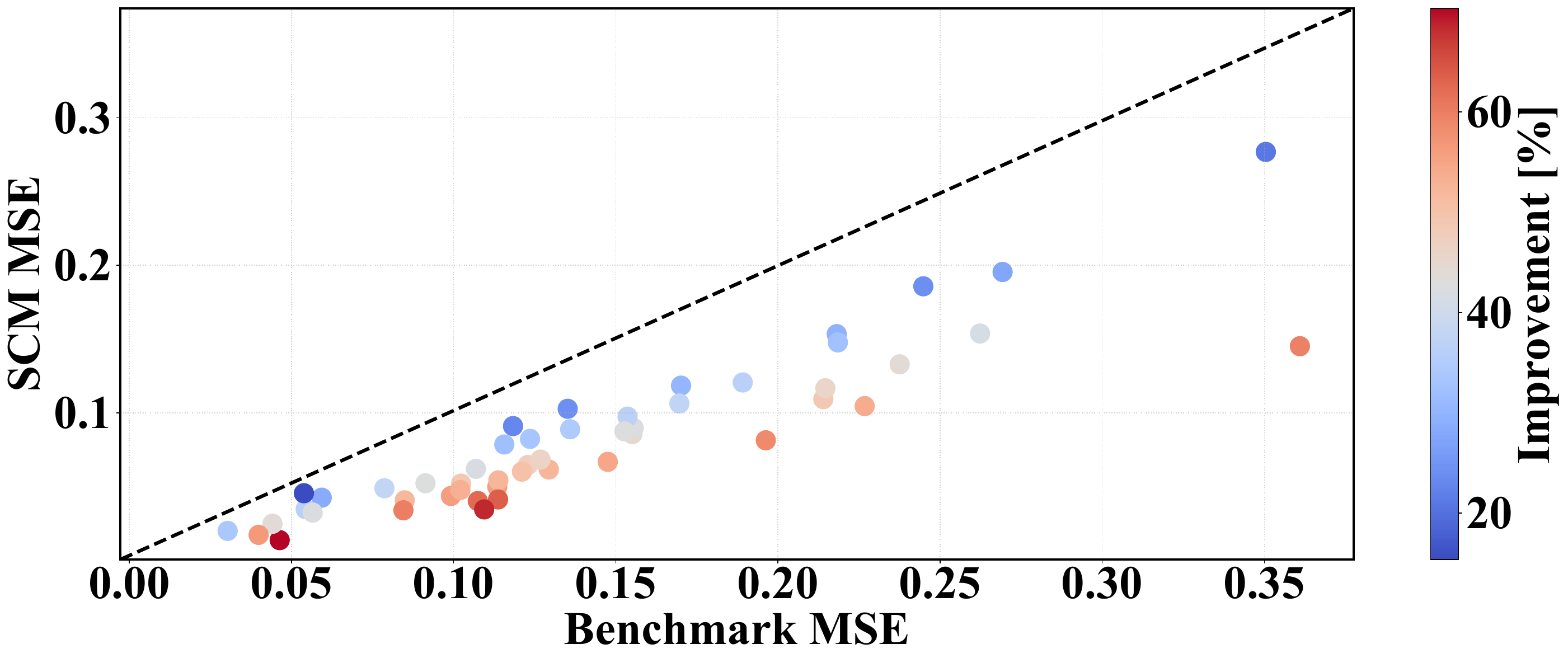}}}
    \caption{MSE Comparison over 50 Buildings between $\underline{\text{SCM}}$+S1+R+Dpast and the Best Benchmark. Sub-Diagonal Points show SCM improvements.}
    \label{fig3_build}
\end{figure}

\begin{table}[h]
\centering
\caption{Accuracy of the Generalized Control Group and the Best Benchmark over 50 Residential Buildings.}
\label{tab3_gCG}
\renewcommand{\arraystretch}{1.2}
\setlength{\tabcolsep}{12.5pt}

\begin{tabular}{|l|c|c|c|c|c|}
\hline
\textbf{Scenario} & \textbf{MSE} & \textbf{Min} & \textbf{Max} & \textbf{STD} & \textbf{Diff} \\ \hline
K-Means + Lasso   & 0.1426 & 0.0303 & 0.3610 & 0.0751 &               \\ 
$\underline{\text{SCM}}$+S1+R           & 0.1274 & 0.0243 & 0.3393 & 0.0815 & 10.66\% \\ 
$\underline{\text{SCM}}$+S1+R+exF     & 0.1271 & 0.0245 & 0.3412 & 0.0817 & 10.87\% \\ 
$\underline{\text{SCM}}$+S1+R+Tpast   & 0.0920 & 0.0155 & 0.2912 & 0.0572 & 35.48\% \\ 
$\underline{\text{SCM}}$+S1+R+Dpast   & 0.0821 & 0.0138 & 0.2768 & 0.0519 & 42.43\% \\ \hline
\end{tabular}%

\vspace{4pt}
\begin{minipage}{\columnwidth}
\scriptsize
\textbf{Abbreviations:} K-Means + Lasso denotes K-Means clustering with Lasso regression; $\underline{\text{SCM}}$ unconstrained SCM; +S1 adds sum-to-one; +R regularization; +exF external features; +Tpast adds treated history; +Dpast adds top-correlated donor history.
\end{minipage}
%\vspace{-10pt}
\end{table}

We posit that accurate baseline estimation benefits from explicitly modeling both external drivers and temporal dependencies, which are not fully captured by the standard \gls{scm} formulation. To demonstrate this, we build on the best-performing linear variant ($\underline{\text{SCM}}$+S1+R) by expanding the control group in three successive stages: incorporating exogenous features (+exF), adding recent values of the treated unit (+Tpast), and including donor-specific lagged values (+Dpast). Each stage builds on the previous one, resulting in progressively richer representations of the control group. \autoref{tab3_gCG} reports the performance of each configuration relative to the strongest benchmark, K-Means + Lasso.

Here, the inclusion of external features (+exF) decreases the mean \gls{mse} from 0.1274 for $\underline{\text{SCM}}$+S1+R to 0.1271, corresponding to a marginal gain over $\underline{\text{SCM}}$+S1+R and a total improvement of \SI{10.87}{\%} relative to the benchmark. Augmenting the model with recent values of the treated unit (+Tpast) further reduces the mean \gls{mse} to 0.0920, yielding an additional improvement of \SI{27.79}{\%} over $\underline{\text{SCM}}$+S1+R and \SI{35.48}{\%} compared to the benchmark. The final extension, incorporating donor-specific lags (+Dpast), attains the lowest mean \gls{mse} of 0.0821, corresponding to improvements of \SI{35.56}{\%} over $\underline{\text{SCM}}$+S1+R and \SI{42.43}{\%} relative to the benchmark. This configuration also reduces the minimum error from 0.0303 to 0.0138, the maximum error from 0.3610 to 0.2768, and the standard deviation from $\pm$0.0751 to $\pm$0.0519, indicating both increased accuracy and reduced variability across buildings.

\autoref{fig3_build} shows building-level \gls{mse} values for the $\underline{\text{SCM}}$+S1+R+Dpast variant compared to the strongest benchmark. Points below the parity line indicate buildings where the \gls{scm} achieves lower errors, with color shading denoting relative improvement. The distribution reveals a consistent advantage for \gls{scm}, with lower \gls{mse} across all 50 buildings. On average, the mean \gls{mse} decreases by \SI{42.43}{\%}, with individual gains ranging from \SI{15.37}{\%} to \SI{70.31}{\%}. Improvements occur across the entire error spectrum, indicating stable performance independent of initial error magnitude.

\begin{figure}
    \centering
    \framebox{\parbox{5 in}{\includegraphics[width=5 in]{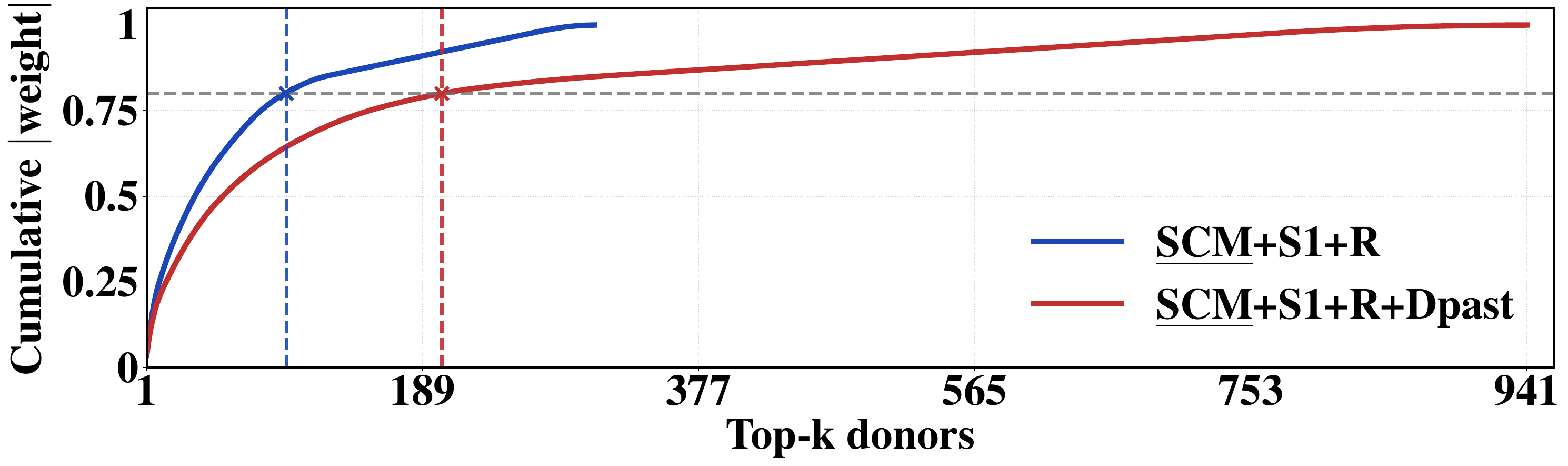}}}
    \caption{Cumulative absolute weight share for $\underline{\text{SCM}}$+S1+R and $\underline{\text{SCM}}$+S1+R+Dpast. Dashed lines indicate the number of top donors required to account for 80\% of the total weight.}
    \label{fig4_w}
\end{figure}

We further analyze the distribution of control group weights to determine whether the observed gains stem from a few dominant donors or from broader participation. \autoref{fig4_w} compares cumulative weight distributions for $\underline{\text{SCM}}$+S1+R and $\underline{\text{SCM}}$+S1+R+Dpast, ranking donors by absolute weight. In $\underline{\text{SCM}}$+S1+R, \SI{80}{\%} of the total absolute weight is concentrated among the top 96 donors, with the largest individual weight of 0.1329. In contrast, $\underline{\text{SCM}}$+S1+R+Dpast requires 202 donors to reach the same \SI{80}{\%} threshold, indicating a more even distribution, though the maximum weight increases slightly to 0.1563. 
The improved accuracy is likely the result of effectively using the additional features, which spreads the weights across more donors.

\subsection{Effect of the Generalized Nonlinear Weighting} \label{sec_results_nonlinear_weights}

To increase \gls{scm} flexibility, the linear weighting is replaced with a nonlinear mapping (DNN) that captures complex dependencies between the treated and donor units. The \gls{dnn}-based approach is evaluated for both the standard and augmented control groups.

\begin{table}[t!]
\centering
\caption{Baseline estimation performance of linear and nonlinear SCM with basic and augmented control groups.}
\label{tab4_gW}
\renewcommand{\arraystretch}{1.2}
\setlength{\tabcolsep}{12.5pt}

\begin{tabular}{|l|c|c|c|c|c|}
\hline
\textbf{Scenario} & \textbf{MSE} & \textbf{Min} & \textbf{Max} & \textbf{STD} & \textbf{Diff} \\ \hline
K-Means+Lasso   & 0.1426 & 0.0303 & 0.3610 & 0.0751 &               \\ 
$\underline{\text{SCM}}$+S1+R           & 0.1274 & 0.0243 & 0.3393 & 0.0815 & 10.66\% \\ 
NN.SCM                                     & 0.1260 & 0.0256 & 0.3636 & 0.0852 & 11.64\% \\ 
$\underline{\text{SCM}}$+S1+R+Dpast   & 0.0821 & 0.0138 & 0.2768 & 0.0519 & 42.43\% \\ 
NN.SCM+Dpast                             & 0.0817 & 0.0127 & 0.2796 & 0.0543 & 42.70\% \\ \hline
\end{tabular}%

\vspace{4pt}
\begin{minipage}{\columnwidth}
\scriptsize
\textbf{Abbreviations:} K-Means + Lasso denotes K-Means clustering with Lasso regression; $\underline{\text{SCM}}$ unconstrained SCM; +S1 adds sum-to-one; +R regularization; NN.SCM neural network extension; +Dpast adds history and top-correlated donor value.
\end{minipage}
%\vspace{-10pt}
\end{table}

\autoref{tab4_gW} shows that both nonlinear and linear variants achieve nearly identical accuracy. For the standard control group, the mean \gls{mse} decreases from 0.1274 (linear) to 0.1260 (nonlinear), a relative improvement of \SI{1.10}{\%}. For the generalized control group, it decreases from 0.0821 to 0.0817, an improvement of \SI{0.49}{\%}. Both variants outperform the benchmark substantially, improving accuracy by \SI{42.43}{\%} (linear) and \SI{42.71}{\%} (nonlinear).

To analyze robustness, we further evaluate both weighting schemes across progressively smaller control groups, as summarized in \autoref{fig5_dec}.

\begin{figure} 
  \centering
  \framebox{\parbox{5 in}{\includegraphics[width=5 in]{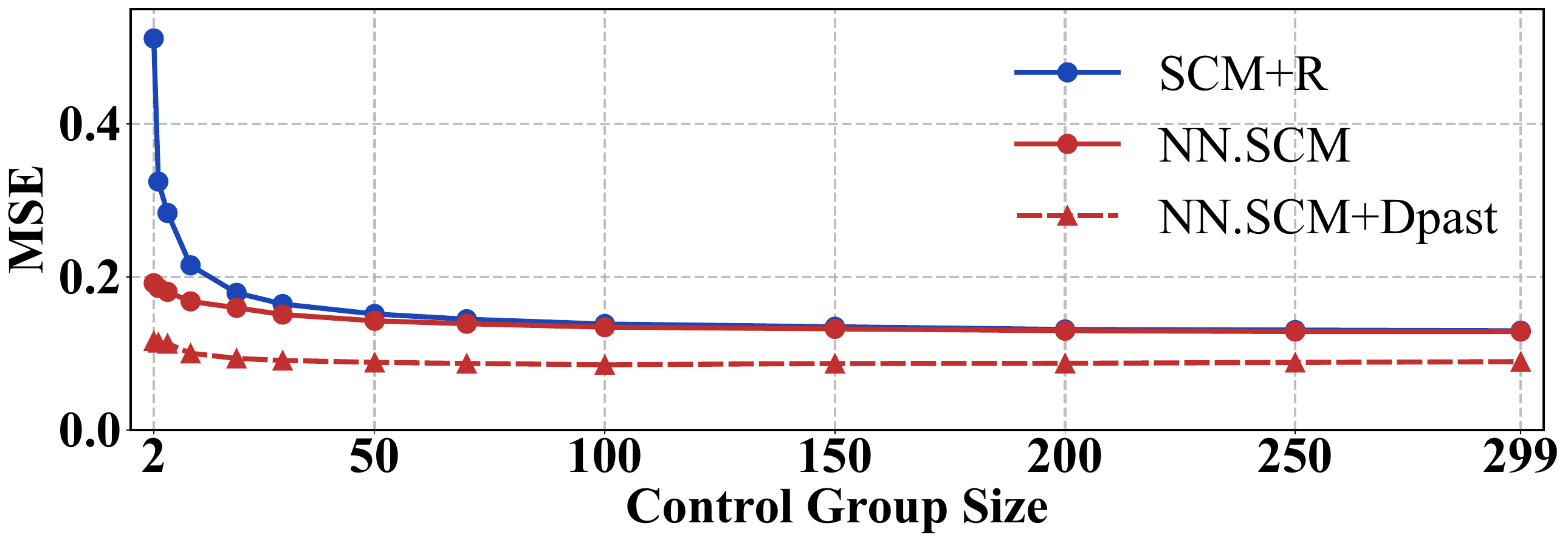}}}
  \caption{Comparison of MSE Between Linear and Nonlinear SCM Variants for Different Control Group Sizes.}
  \label{fig5_dec}
\end{figure}

For both linear and nonlinear \gls{scm} variants, expanding the donor pool consistently lowers the \gls{mse}. 
As the donor pool shrinks, performance gaps widen. 
%In the extreme case of two donors, the error increases sharply to 0.5116 for SCM+R but remains considerably lower at 0.1918 for NN.SCM and 0.1163 for NN.SCM+Dpast. 
These results suggest that while both methods perform comparably with large donor pools, the nonlinear \gls{scm} maintains lower errors and reduced variability when the control group is small.

\section{DISCUSSION AND LIMITATIONS} \label{sec_discussion}
%The proposed \gls{scm} consistently outperformed all benchmark methods across all the evaluated scenarios.

In \autoref{sec_results_linear}, the standard linear \gls{scm} outperformed the benchmarks. Unlike clustering-based methods that preselect donors by similarity, the \gls{scm} optimizes weights over the entire donor pool, allowing the model to combine donors that may be individually dissimilar to the treated unit but jointly approximate its behavior well. Relaxing the non-negativity constraint further allows it to exploit both positive and negative demand correlations, increasing flexibility in reconstructing counterfactual load profiles.

\autoref{sec_results_generalized_control_group} extends the control group with external features, recent treated unit history, and lagged donor values. Weather and calendar features yielded only limited gains. This could be partially attributed to the low temporal resolution of the weather data. Lagged treated values provide an autoregressive signal that compensates for imperfect donor matches. Lagged donor values improving matching by aligning similar but temporally misaligned patterns. Further discussion is supplied in the Appendix.

In \autoref{sec_results_nonlinear_weights}, the linear \gls{scm} matches the generalized (nonlinear) SCM when the donor pool is large, suggesting that linearity is sufficient in data-rich settings. When the donor pool is small or less representative, however, the nonlinear variant achieves lower errors and reduced variance. 
%The nonlinear weighting thus increases robustness in data-constrained scenarios.

However, some limitations remain and suggest directions for further improvement. The evaluation uses a single dataset that may not capture the full range of regulatory settings or flexibility mechanisms. Benchmarking also remains challenging because existing methods rely on different assumptions, such as prior knowledge of activation periods, and therefore require careful adaptation for fair comparison.

\section{Conclusions and Future Work}\label{sec_conclusion}
We proposed a Generalized Synthetic Control framework for baseline estimation in residential demand response, extending the classical \gls{scm} with an augmented control representation and a nonlinear weighting function. On the Ausgrid dataset, the linear \gls{scm} outperformed the strongest benchmarks by more than \SI{10}{\%}, while the generalized variant improved accuracy by up to \SI{42}{\%}. Linear weighting remained highly effective with large donor pools, whereas the nonlinear model increased robustness and accuracy for smaller donor pools. Future work includes uncertainty quantification for counterfactual trajectories, multi-step recursive prediction, broader validation beyond Ausgrid, and more principled donor and lag selection. Additional directions include stronger nonlinear models and robustness to distribution shift and strategic behavior.

\appendices

% This line forces the word 'Appendix' and a space before the actual letter/number
\makeatletter
\def\p@section{Appendix~}
\def\p@subsection{Appendix~}
\makeatother

\section{Experimental Setup} \label{ap:scm_exp_setup}
We use the Ausgrid dataset \cite{ratnam_residential_2017}, comprising half-hourly measurements of electrical load and \gls{pv} generation (in kW) from 300 Australian households, recorded from 2010 to 2013. To enhance the dataset, we simulate \gls{bess} operation using a linear programming approach (Appendix~\ref{ap:LP_bess}). The optimization determines economically optimal charging and discharging schedules under dynamic electricity prices, resulting in comprehensive prosumption time series that integrate household load, \gls{pv} generation, and \gls{bess} behavior. To capture exogenous influences, we add meteorological variables from Meteostat \cite{meteostat_sydney_nodate}, aligned with the Ausgrid data in space and time. The resulting dataset is split into training (\SI{60}{\%}), validation (\SI{10}{\%}), and testing (\SI{30}{\%}) partitions and is publicly available at \cite{sievers_github_2025}.

All methods are compared across 50 buildings, each evaluated as the treated unit in turn against the remaining $J=299$ donors. The augmented control group further adds $L=336$ recent treated-unit values for $\mathbf{H}$, one selected lagged sequence per donor for $\mathbf{R}$, and $M=7$ exogenous features for $\mathbf{F}$: four meteorological variables (temperature, humidity, wind speed, wind direction) and three temporal indicators (weekday, and sine and cosine of hour-of-day). 

For the linear \gls{scm}, the best results are obtained with $\ell_2$ regularization, using $\lambda=450$ for the augmented and $\lambda=1350$ for the basic control group to promote uniform weights.
For the nonlinear \gls{scm}, we test multiple \gls{mlp} configurations on the validation set, varying model depth, width, learning rate, and batch size. The best-performing model uses two hidden layers of 64 units, followed by a layer of $J$ units and a scalar output. Training employs a learning rate of $5\times10^{-4}$, early stopping with a patience of 10, and a maximum of 2000 epochs. \gls{mse} values are averaged over all treated buildings, and neural networks are trained three times per building to reduce variability from random initialization.

Benchmark models are adapted from their original publications, with architectures and hyperparameters tuned to the dataset and baseline estimation task. Additional implementation details can be found in the cited works.

\section{Linear Program for Energy Management System} \label{ap:LP_bess}
This section outlines the linear program used to optimize \gls{bess} operations under dynamic electricity prices from the Australian Energy Market Operator \cite{aemo_aggregated_nodate}. The \gls{ems} minimizes total electricity costs, accounting for both consumption and flexibility revenues.

\emph{Objective Function:}
\begin{equation}
    \min \sum_{t \in T} \left( p_t \cdot x_t - f_t \cdot \delta_t^{+} \right)
\end{equation}
with dynamic electricity price $p_t$, net grid consumption $x_t$, flexibility rebate $f_t$, and rewarded reduction $\delta_t^+$.

\emph{Constraints:} \\
Net grid consumption $x_t$ includes building load $l_t$, \gls{pv} generation $g_t$, and \gls{bess} operation $s_t$.
\begin{equation}
    x_t = l_t - g_t + s_{t}, \quad \forall t \in T
\end{equation}

\pagebreak

The baseline $b_t$ is the average consumption during the same hour of the $n$ most recent congestion-free days:

\begin{equation}
    b_t = \dfrac{1}{n} \sum_{k \in S_t} x_{t - 24k}, \quad \forall t \in T,
    \quad \forall k \in \{1,\dots,n\}
\end{equation}
\begin{equation}
    S_t = \left\{ k \in \{1, 2, \dots, n\} \ \bigg| \ \sum_{h=-12}^{12} C_{t - 24k + h} = 0 \right\}
\end{equation}
         
Flexibility $\delta_t^+$ is the positive deviation from the baseline.
\begin{equation}
    \delta_t^+ = max(0, b_t - x_t), \quad \forall t \in T
\end{equation}

The \gls{soc} is updated with dis-charging actions $s_t$, considering efficiency $\eta_{s}$.

\begin{equation}
   \text{SOC}_t = \text{SOC}_{t-1} + \eta_{\text{s}} \cdot s_{t}, \forall t \in T \setminus \{0\}.
\end{equation}
    
Domain constraints ensure feasible operation:
\begin{equation}
s_{t} \in [-s^{max}, s^{max}], \quad \text{soc}_t \in [soc^{min}, soc^{max}]
\end{equation}

\section{Theoretical Justification of Augmented SCM} \label{ap:scm_theory}

This appendix gives a stylized justification for including lagged treated outcomes in the augmented Synthetic Control design. For brevity, we isolate one mechanism: under a linear factor model with exact contemporaneous factor replication by the donor pool, a lagged treated outcome can reduce one-step counterfactual mean-squared error by correcting persistent idiosyncratic mismatch left by static SCM. Analysis can be extended to the full  ridge estimator used in the paper at the expense of extra notation and modeling assumptions. 

\subsection*{A.1 Setup and assumptions}

We adopt the \emph{linear factor model} common in SC literature \cite{abadie_synthetic_2010}. For \(t \in T_0 \cup T_1\), let
\begin{equation}
y_{jt} = \lambda_j^\top f_t + \varepsilon_{jt},
\end{equation}
where \(f_t \in \mathbb{R}^r\) is a common-factor process, \(\lambda_j \in \mathbb{R}^r\) is the loading vector of unit \(j\), and \(\varepsilon_{jt}\) is an idiosyncratic component. Let
\begin{equation}
SC_t := \sum_{j=2}^{J+1} w_j^* y_{jt}
\end{equation}
denote the standard synthetic control constructed from donor weights \(w^*=(w_2^*,\dots,w_{J+1}^*)^\top\), estimated on the pre-intervention period. We impose the following assumptions.

\noindent\paragraph*{\bf{Assumptions}}
\begin{itemize}
\item[\bf{A1}][Exact factor replication]
There exist donor weights \(w_j^*\) satisfying \(\sum_{j=2}^{J+1} w_j^* = 1\) such that
\begin{equation}
\lambda_1 = \sum_{j=2}^{J+1} w_j^* \lambda_j .
\end{equation}

\item[\bf{A2}][AR(1) treated idiosyncratic component]
The treated idiosyncratic component satisfies
\begin{equation}
\varepsilon_{1t} = \rho \varepsilon_{1,t-1} + \nu_t, \qquad |\rho|<1,
\end{equation}
where \(\{\nu_t\}\) is white noise with variance \(\sigma_\nu^2\). Hence the stationary variance of \(\varepsilon_{1t}\) is

\begin{equation}
\sigma_\varepsilon^2 := \operatorname{Var}(\varepsilon_{1t}) = \frac{\sigma_\nu^2}{1-\rho^2}.
\end{equation}

\item[\bf{A3}][Donor idiosyncratic noise]
For \(j=2,\dots,J+1\), the donor idiosyncratic components \(\varepsilon_{jt}\) are mutually independent, independent of \(\{\varepsilon_{1t}\}_{t}\), serially uncorrelated, and satisfy $\operatorname{Var}(\varepsilon_{jt}) = \sigma_j^2.$ The aggregated donor noise and its variance are given by:
\begin{equation}
\eta_t := \sum_{j=2}^{J+1} w_j^* \varepsilon_{jt},
\enspace  
\omega^2 := \operatorname{Var}(\eta_t) = \sum_{j=2}^{J+1} (w_j^*)^2 \sigma_j^2.
\end{equation}
\item[\bf{A4}][Stochastic factors]
\label{ass:factors}
The factor process \(\{f_t\}\) is stationary with mean zero and covariance
$
\Sigma_f := \operatorname{Var}(f_t),
$
and is independent of all idiosyncratic components. Define
\begin{equation}
\sigma_\mu^2 := \operatorname{Var}(\lambda_1^\top f_t) = \lambda_1^\top \Sigma_f \lambda_1.
\end{equation}
\end{itemize}

\subsection*{A.2 A stylized lag-augmented one-step estimator}

To isolate the role of a lagged treated outcome, consider the one-step estimator
\begin{equation}
\widehat y^{\,LAG}_{1,T_0+1}(\beta)
:=
SC_{T_0+1} + \beta\, y_{1,T_0},
\end{equation}
where \(\beta \in \mathbb{R}\) is a scalar coefficient. This is a simplified surrogate for the paper's augmented estimator. 
%: it corresponds to adding the first lag of the treated outcome while keeping the contemporaneous donor weights fixed at \(w^*\).

\begin{proposition}[MSE reduction from lag augmentation]
\label{prop:main}
Under Assumptions {A1-A4}, define
\begin{equation}
A := \frac{\sigma_\varepsilon^2}{\sigma_\varepsilon^2 + \sigma_\mu^2} \in (0,1].
\end{equation}
Then:
\begin{itemize}
    \item[(i)] The population-optimal coefficient is 
    \begin{equation}
    \beta^* = \rho A.
    \end{equation}
    \item[(ii)] The resulting optimal one-step MSE is
    \begin{equation}
    \operatorname{MSE}(\beta^*)
    =
    (\sigma_\varepsilon^2 + \omega^2)
    - \rho^2 \sigma_\varepsilon^2 A.
    \end{equation}
    \item[(iii)] Hence the MSE reduction relative to standard SCM is
    \begin{equation}
    \Delta
    :=
    \operatorname{MSE}(SC) - \operatorname{MSE}(\beta^*)
    =
    \rho^2 \sigma_\varepsilon^2 A.
    \end{equation}
    In particular, the lag-augmented estimator strictly improves on standard SCM if and only if \(\rho \neq 0\).
\end{itemize}
\end{proposition}

\begin{proof}
Using Assumption~A1,
\[
SC_{T_0+1} = \lambda_1^\top f_{T_0+1} + \eta_{T_0+1},
\]
while
\[
y_{1,T_0+1} = \lambda_1^\top f_{T_0+1} + \varepsilon_{1,T_0+1}.
\]
Therefore,
\begin{align*}
e^{LAG}_{T_0+1}(\beta)
&=
y_{1,T_0+1} - \widehat y^{\,LAG}_{1,T_0+1}(\beta)\\[0.05in]
&=
\varepsilon_{1,T_0+1} - \eta_{T_0+1} - \beta y_{1,T_0}\\[0.05in]
&=
    (\varepsilon_{1,T_0+1} - \beta \varepsilon_{1,T_0})
    - \eta_{T_0+1}
    - \beta \lambda_1^\top f_{T_0}.
\end{align*}

\noindent By Assumptions~A2--A4, the three terms
\[
\varepsilon_{1,T_0+1} - \beta \varepsilon_{1,T_0},
\qquad
\eta_{T_0+1},
\qquad
\lambda_1^\top f_{T_0}
\]
are mutually uncorrelated, hence
\[
\operatorname{MSE}(\beta)
=
\mathbb{E}\!\left[(\varepsilon_{1,T_0+1}-\beta\varepsilon_{1,T_0})^2\right]
+
\omega^2
+
\beta^2 \sigma_\mu^2.
\]
Expanding the first term gives
\[
\mathbb{E}\!\left[(\varepsilon_{1,T_0+1}-\beta\varepsilon_{1,T_0})^2\right] \\
=
\sigma_\varepsilon^2 - 2\beta \operatorname{Cov}(\varepsilon_{1,T_0+1},\varepsilon_{1,T_0}) + \beta^2 \sigma_\varepsilon^2,
\]
with
\[
\operatorname{Cov}(\varepsilon_{1,T_0+1},\varepsilon_{1,T_0}) = \rho \sigma_\varepsilon^2,
\]
under Assumption A2. Therefore,
\[
\operatorname{MSE}(\beta)
=
(\sigma_\varepsilon^2+\omega^2) - 2\beta\rho\sigma_\varepsilon^2 + \beta^2(\sigma_\varepsilon^2+\sigma_\mu^2).
\]
This is a strictly convex quadratic function in \(\beta\), so minimizing gives
\[
\beta^* = \frac{\rho \sigma_\varepsilon^2}{\sigma_\varepsilon^2 + \sigma_\mu^2} = \rho A,
\]
which establishes part (i).
Substituting \(\beta^*\) into the quadratic yields
\[
\operatorname{MSE}(\beta^*)=
(\sigma_\varepsilon^2+\omega^2)
-
\frac{(\rho \sigma_\varepsilon^2)^2}{\sigma_\varepsilon^2+\sigma_\mu^2} \\
= 
(\sigma_\varepsilon^2+\omega^2)-\rho^2\sigma_\varepsilon^2 A,
\]
which gives part (ii). Part (iii) follows because under Assumption A1, the standard SCM error is
\[
e_t^{SC} := y_{1t} - SC_t = \varepsilon_{1t} - \eta_t,
\]
and thus,
\[
\operatorname{MSE}(SC) = \operatorname{Var}(e_t^{SC}) = \sigma_\varepsilon^2 + \omega^2.
\]
\end{proof}

\begin{remark}
The coefficient \(\beta^*=\rho A\) is attenuated relative to \(\rho\) because the lagged treated outcome \(y_{1,T_0}\) mixes the predictive idiosyncratic signal \(\varepsilon_{1,T_0}\) with the factor component \(\lambda_1^\top f_{T_0}\). Thus, the lagged treated value is a proxy for the target residual signal.
\end{remark}

\subsection*{A.3 Residual-Feedback benchmark}

For comparison, consider the following residual-correcting feedback mechanism:
\begin{equation}
\widehat y^{\,RC}_{1,T_0+1}
:=
SC_{T_0+1} + \rho e_{T_0}^{SC},
\end{equation}
where $e_{T_0}^{SC} = y_{1,T_0}-SC_{T_0}$. 
This mechanism uses the latent SCM residual \(e_{T_0}^{SC}\) directly.

\begin{proposition}[Residual-feedback benchmark]
\label{prop:residual}
Under Assumptions {{A1-A4}}, the residual-feedback benchmark satisfies
\begin{equation}
e^{RC}_{T_0+1}
=
(\varepsilon_{1,T_0+1} - \rho \varepsilon_{1,T_0})
-
(\eta_{T_0+1} - \rho \eta_{T_0}),
\end{equation}
and
\begin{equation}
\operatorname{MSE}(RC)
=
(1-\rho^2)\sigma_\varepsilon^2 + (1+\rho^2)\omega^2.
\end{equation}
Hence, its MSE reduction relative to standard SCM is
\begin{equation}
\Delta_{RC} = \rho^2(\sigma_\varepsilon^2 - \omega^2).
\end{equation}
The residual-feedback benchmark thus improves on standard SCM if and only if \(\rho \neq 0\) and \(\sigma_\varepsilon^2 > \omega^2\).
\end{proposition}

\begin{proof}
Under Assumption A1 the error of the standard SCM is \(e_{T_0}^{SC}=\varepsilon_{1,T_0}-\eta_{T_0}\). Therefore,
\[
e^{RC}_{T_0+1}
=
e^{SC}_{T_0+1} - \rho e^{SC}_{T_0} \\
=
(\varepsilon_{1,T_0+1}-\rho\varepsilon_{1,T_0})
-
(\eta_{T_0+1}-\rho\eta_{T_0}).
\]
By Assumption~A2, we have
\[
\operatorname{Var}(\varepsilon_{1,T_0+1}-\rho\varepsilon_{1,T_0})
=
\operatorname{Var}(\nu_{T_0+1})
=
(1-\rho^2)\sigma_\varepsilon^2.
\]
Furthermore, since the donor noise is serially uncorrelated (Assumption A3),
we have 
\[
\operatorname{Var}(\eta_{T_0+1}-\rho\eta_{T_0})
=
(1+\rho^2)\omega^2.
\]
Adding the two terms gives the equation for MSE($RC$). Recall that under Assumption A1, the standard SCM error is
\[
e_t^{SC} := y_{1t} - SC_t = \varepsilon_{1t} - \eta_t,
\]
and thus,
\[
\operatorname{MSE}(SC) = \sigma_\varepsilon^2 + \omega^2.
\]
Subtracting MSE($SC$) from MSE($RC$) gives
\[
\Delta_{RC} = \rho^2(\sigma_\varepsilon^2-\omega^2).
\]
\end{proof}

\begin{remark}
The residual-feedback mechanism removes factor contamination by correcting with \(e_{T_0}^{SC}\) rather than \(y_{1,T_0}\), but it also feeds back lagged donor noise through \(-\rho \eta_{T_0}\). When donor noise is sufficiently small, this can outperform the lag-treated correction of Proposition~\ref{prop:main}. When donor noise is larger, however, that feedback can outweigh the benefit. In particular, for \(\rho\neq 0\), we have
\begin{equation}
\Delta_{RC} > \Delta
\iff
\omega^2 < \frac{\sigma_\varepsilon^2\sigma_\mu^2}{\sigma_\varepsilon^2+\sigma_\mu^2},
\end{equation}
with equality at the threshold.
\end{remark}

Proposition~\ref{prop:main} analyzes only the one-step surrogate obtained by adding a single lagged treated outcome while holding the contemporaneous donor weights fixed. Its purpose is to motivate the inclusion of persistent treated-unit mismatch not captured by static SCM.

The residual-feedback benchmark of Proposition \ref{prop:residual} clarifies a second point. Since the predictor 
\begin{equation}
SC_t + \rho e_{t-1}^{SC}
=
SC_t + \rho y_{1,t-1} - \rho \sum_{j=2}^{J+1} w_j^* y_{j,t-1},
\end{equation}
is linear in the lagged treated outcome and the first lag of every donor outcome, the benchmark is \emph{representable} inside an augmented linear feature class only if both the lagged treated outcome \(y_{1,t-1}\), and the full first-lag donor block \(\{y_{j,t-1}\}_{j=2}^{J+1}\) are included. Joint
linear or nonlinear re-optimization can then in principle, learn a
predictor that improves on both the lag-treated surrogate
of Proposition 1 and the residual-feedback mechanism of Proposition 2,
provided the additional flexibility is supported by the data.

\balance
%%%%%%%%%%%%%%%%%%%%%%%%%%%%%%%%%%%%%%%%%%%%%%%
\bibliographystyle{IEEEtran}
\bibliography{references}

\end{document}